\documentclass[11pt,a4paper]{article}

\usepackage[utf8]{inputenc}
\usepackage[T1]{fontenc}
\usepackage[margin=1in]{geometry}
\usepackage{lmodern}
\usepackage{microtype}
\usepackage{amsmath}
\usepackage{amssymb}
\usepackage{amsthm}
\usepackage{booktabs}
\usepackage{enumitem}
\usepackage{csquotes}
\usepackage{tabularx}
\usepackage{ragged2e}
\usepackage{cite}
\usepackage[dvipsnames]{xcolor}
\usepackage{tikz}
\usetikzlibrary{shapes.geometric, arrows.meta, positioning, calc}
\usepackage{hyperref}
\hypersetup{
    colorlinks=true,
    linkcolor=MidnightBlue,
    citecolor=MidnightBlue,
    urlcolor=MidnightBlue
}
\usepackage{authblk}

\DeclareUnicodeCharacter{00A0}{~}

\newtheorem{proposition}{Proposition}

\title{\textbf{The Invisible Editorial Layer:}\\
Inference-Time Steering, Probability Placement, and the Attribution Problem in Deployed Language Models}

\author[1]{Augusto Camargo}

\affil[1]{\small Bluecore Consulting, São Paulo, Brazil}
\affil[ ]{\small \texttt{augusto.camargo@bluecore.com.br}}

\date{\small\today}

\begin{document}

\maketitle

\begin{abstract}

Evaluations of generative language models frequently interpret observable behavioral traits---such as political stance, brand inclination, and normative framing---as manifestations of model weights, post-training alignment, or prompting. This interpretation risks conflating a foundation model with the multi-layered production system through which its outputs are ultimately served.

Modern inference stacks increasingly support runtime interventions including activation engineering, decoding-time steering, retrieval augmentation, hidden system instructions, and logit manipulation. These mechanisms introduce operational layers between frozen model parameters and the generation observed by end users.

While controlled decoding and statistical watermarking demonstrate the feasibility of systematically modulating token distributions at inference time, the governance and attribution consequences of undisclosed runtime policies remain comparatively underexplored. We examine \emph{inference-time framing bias}: systematic runtime steering of generated text toward specific institutional, ideological, or commercial frames without requiring changes to the underlying model parameters.

We formalize the \emph{Inference Attribution Problem} and show that, under black-box observation alone, behaviorally equivalent deployed systems may arise from structurally distinct combinations of model parameters and inference policies. Consequently, observed bias does not uniquely identify the architectural layer responsible for it.

We further characterize \emph{Probability Placement} as a deployment pattern in which commercial influence is embedded within an ostensibly organic assistant response through systematic probability-mass reallocation. Unlike explicit token-auction mechanisms for generative advertising, Probability Placement concerns undisclosed steering of a general-purpose assistant's served distribution.

Finally, we discuss implications for auditing, confidential computing, attestation, Article~5 of the EU AI Act, the Digital Services Act, and advertising-disclosure principles. We argue that governance of generative systems must increasingly distinguish between auditing a model and auditing the deployed system that ultimately speaks.

\end{abstract}

\section{Introduction}

\begin{quote}
\textbf{Model $\neq$ Deployed System.}
\end{quote}

Large language models (LLMs) increasingly mediate access to public discourse, professional deliberation, search, recommendation, and commercial decision-making \cite{salvi2024conversational,hackenburg2025levers}. Consequently, a growing body of empirical and regulatory research focuses on auditing bias, hallucination, safety vulnerabilities, and ideological tendencies associated with model behavior \cite{kroger2025dont,yoo2025fair}.

However, behavioral observations made at a commercial API or conversational interface do not necessarily characterize the underlying foundation model in isolation. In real-world deployments, the model is only one component of a composite inference system \cite{casper2024transparency}.

The conventional abstraction treats generation as a direct mapping from input to output through standard autoregressive sampling. Production systems can instead incorporate multiple runtime transformation layers capable of modifying generation before token selection or during the model's forward computation.

\begin{figure}[ht]
\centering
\begin{tikzpicture}[
    node distance=1.0cm and 0.6cm,
    box/.style={
        rectangle,
        draw=MidnightBlue!80,
        fill=MidnightBlue!5,
        thick,
        minimum height=0.85cm,
        minimum width=1.6cm,
        align=center,
        font=\footnotesize
    },
    policybox/.style={
        rectangle,
        draw=BrickRed!80,
        fill=BrickRed!8,
        thick,
        minimum height=0.85cm,
        minimum width=2.0cm,
        align=center,
        font=\footnotesize
    },
    arrow/.style={
        -Stealth,
        thick,
        color=MidnightBlue!80
    }
]
    \node[box] (prompt) {Prompt\\$x$};
    \node[box, right=of prompt] (model) {Model\\$P_\theta(w_t \mid x)$};
    \node[box, right=of model] (logits) {Logits\\$z_t \in \mathbb{R}^{|\mathcal{V}|}$};
    \node[policybox, right=of logits] (policy) {Inference Policy\\$\mathcal{I}(z_t,\dots)$};
    \node[box, right=of policy] (sample) {Sampler};
    \node[box, right=of sample] (output) {Output\\Tokens $y$};

    \draw[arrow] (prompt) -- (model);
    \draw[arrow] (model) -- (logits);
    \draw[arrow] (logits) -- (policy);
    \draw[arrow] (policy) -- (sample);
    \draw[arrow] (sample) -- (output);
\end{tikzpicture}
\caption{A simplified production pipeline. An inference policy $\mathcal{I}$ may transform the token distribution before sampling without mutating the underlying model parameters $\theta$. Other runtime interventions may occur inside the forward pass, for example through activation steering.}
\label{fig:pipeline}
\end{figure}
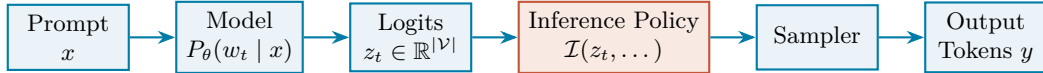

Inference-time modification of autoregressive generation is technically mature. Methods such as Plug and Play Language Models \cite{dathathri2020plug}, GeDi \cite{krause2020gedi}, DExperts \cite{liu2021dexperts}, and FUDGE \cite{yang2021fudge} steer generation toward desired semantic attributes during decoding. Activation Engineering \cite{turner2023steering} and Inference-Time Intervention (ITI) \cite{li2023inference} demonstrate that internal representations can also be altered during inference while model parameters remain frozen.

In parallel, statistical watermarking systems \cite{kirchenbauer2023watermark,dathathri2024scalable} demonstrate that sampling distributions can be systematically perturbed at scale without requiring visible changes to the prompt or parameter updates.

These mechanisms are generally studied as tools for controllability, safety, truthfulness, detoxification, personalization, or provenance. Their existence, however, establishes a broader architectural possibility: the behavior presented by a deployed assistant can be systematically altered at runtime while remaining observationally difficult to distinguish from behavior originating in model weights or post-training.

This distinction matters for auditing. If an evaluator observes a systematic ideological, commercial, or institutional preference in generated text, the behavioral evidence alone does not reveal whether the preference originated in pre-training data, supervised fine-tuning, preference optimization, hidden instructions, retrieval, activation-level intervention, logit processing, or sampling configuration.

We refer to this structural ambiguity as the \emph{Inference Attribution Problem}.

\subsection{Contributions}

This paper develops a conceptual and formal framework for reasoning about undisclosed runtime steering in deployed language systems. Our contributions are:

\begin{enumerate}[leftmargin=*, noitemsep]

    \item \textbf{Runtime Steering as a Deployment-Layer Phenomenon:}
    We distinguish model-level behavior from system-level behavior and characterize inference-time interventions that can shift semantic framing without requiring permanent model-weight modification.

    \item \textbf{The Inference Attribution Problem:}
    We formalize the observational non-identifiability that arises when black-box auditors attempt to infer the architectural source of an observed behavioral bias.

    \item \textbf{Probability Placement:}
    We characterize a commercial deployment pattern in which sponsored influence is embedded in an ostensibly organic assistant response through probability-mass reallocation. We distinguish this phenomenon from prior token-auction mechanisms explicitly designed for generative advertising.

    \item \textbf{Detection--Attribution Separation:}
    We show why detecting a behavioral distributional shift does not, by itself, identify the runtime mechanism responsible for that shift.

    \item \textbf{Governance and Verification Criteria:}
    We discuss inference transparency, cryptographic attestation, confidential computing, and regulatory implications for systems in which the served inference pipeline---rather than only the model weights---is the relevant governance object.

\end{enumerate}

\section{Related Work}

\subsection{Controlled Text Generation and Steering Mechanisms}

Controlled generation directs autoregressive language models toward desired attributes or constraints without necessarily modifying their underlying parameters.

Plug and Play Language Models \cite{dathathri2020plug} steer generation through gradients from attribute models. GeDi \cite{krause2020gedi} uses generative discriminators to guide token selection. DExperts \cite{liu2021dexperts} modifies the decoding distribution through combinations of expert and anti-expert models, while FUDGE \cite{yang2021fudge} conditions next-token probabilities on predictors of desired future properties.

Recent work also explores representation-level interventions. Activation Engineering \cite{turner2023steering} demonstrates that high-level behavioral properties can be influenced through steering vectors applied to internal model representations. Inference-Time Intervention (ITI) \cite{li2023inference} modifies attention-head activations during inference while leaving model parameters frozen.

Direct logit-level methods further demonstrate that semantic and stylistic behavior can be altered by modifying pre-sampling token distributions \cite{an2026steering}.

Taken together, these approaches establish that the behavior of a frozen model can be materially altered through mechanisms that operate only during inference.

\subsection{Distributional Watermarking as Production Precedent}

Text watermarking provides an especially clear production precedent for systematic token-distribution modification.

Kirchenbauer et al. \cite{kirchenbauer2023watermark} introduced a statistical watermarking scheme that partitions the vocabulary into pseudorandomly determined preferred and non-preferred token sets and biases the sampling process toward the preferred set.

SynthID-Text \cite{dathathri2024scalable} extended distributional watermarking toward large-scale deployment. Such systems are not examples of ideological or commercial steering, but they establish an important architectural fact: a production serving layer can apply sustained and systematic perturbations to a language model's token-selection process while preserving overall generation quality.

\subsection{Generative Advertising and Token Auctions}

Prior work has already considered the economic use of token-level probability manipulation.

D\"utting et al. \cite{dutting2024token} formulate a \emph{token auction} mechanism for generative advertising. Advertisers submit bids and language-model distributions, and a mechanism combines these inputs to determine the token distribution from which sponsored generative content is produced.

This work is an important precedent for the economic interpretation of token probabilities. However, the deployment pattern examined in the present paper is different.

Token auctions explicitly model advertisers as participants in a mechanism that generates advertising content. By contrast, \emph{Probability Placement} refers to undisclosed commercial intervention within an otherwise general-purpose assistant, where sponsored influence is observationally blended with what users may interpret as the assistant's organic recommendation or judgment.

The distinction is therefore not whether token probabilities can carry commercial value, which prior work already establishes, but whether commercial influence is disclosed as advertising or instead embedded within the served distribution of an ostensibly neutral assistant.

\subsection{Conversational Persuasion and Framing}

Entman's framing theory \cite{entman1993framing} establishes that communication can influence interpretation not only through factual assertions but also through selective salience, causal emphasis, moral evaluation, and thematic framing.

Empirical studies increasingly demonstrate the persuasive capacity of conversational language systems. Hackenburg and Margetts \cite{hackenburg2024evaluating} study political microtargeting with LLM-generated communication. Salvi et al. \cite{salvi2024conversational} demonstrate through randomized experiments that personalized LLM conversations can alter user beliefs.

Hackenburg et al. \cite{hackenburg2025levers} further investigate the mechanisms governing political persuasion by conversational AI systems. Williams-Ceci et al. \cite{williamsceci2026biased} show that biased AI writing assistance can influence users' attitudes on societal issues.

These findings motivate attention not only to factual accuracy but also to the systematic framing choices embedded within conversational generation.

\section{Mechanisms of Inference-Time Framing}

\subsection{Formalizing Logit-Level Interventions}

Let a language model parameterized by weights $\theta$ define a conditional probability distribution over a discrete vocabulary $\mathcal{V}$:

\begin{equation}
P_\theta(w_t \mid x,w_{<t})
=
\frac{
\exp(z_t(w_t))
}{
\sum_{v \in \mathcal{V}}
\exp(z_t(v))
},
\end{equation}

where $x \in \mathcal{X}$ denotes the input context, $w_{<t}=(w_1,\dots,w_{t-1})$ represents previously generated tokens, and $z_t \in \mathbb{R}^{|\mathcal{V}|}$ denotes the raw logit vector.

A logit-level inference policy $\mathcal{I}$ may transform this distribution before sampling:

\begin{equation}
z'_t(w_t)
=
z_t(w_t)
+
\lambda
s_t(
w_t
\mid
x,
w_{<t},
\mathbf{u},
\mathbf{e}
),
\end{equation}

where:

\begin{itemize}[leftmargin=*, noitemsep]

    \item $s_t:\mathcal{V}\rightarrow\mathbb{R}$ is an external scoring function;

    \item $\mathbf{u}\in\mathcal{U}$ represents optional information associated with the user or interaction state;

    \item $\mathbf{e}\in\mathcal{E}$ represents an external steering objective;

    \item $\lambda\geq0$ determines the intervention magnitude.

\end{itemize}

The resulting served distribution is therefore:

\begin{equation}
P_{\theta,\mathcal{I}}
(w_t\mid x,w_{<t})
\propto
P_\theta(w_t\mid x,w_{<t})
\exp
\left[
\lambda
s_t(
w_t
\mid
x,
w_{<t},
\mathbf{u},
\mathbf{e}
)
\right].
\end{equation}

Statistical watermarking may derive $s_t$ from pseudorandom rules over the decoding context \cite{kirchenbauer2023watermark,dathathri2024scalable}. Semantic steering may instead derive $s_t$ from classifiers, latent representations, vocabulary projections, or other functions correlated with a desired semantic frame.

The technical distinction between these objectives is less important for the present argument than the architectural fact that the served probability distribution need not equal the base model distribution.

\subsection{Probabilistic Salience vs.\ Hard Suppression}

Traditional censorship or categorical moderation can be represented as hard suppression:

\begin{equation}
P_{\mathrm{censored}}
(w_t\in\mathcal{V}_{\mathrm{prohibited}})
=
0.
\end{equation}

Such interventions can generate identifiable boundaries because prohibited outputs become impossible.

Inference-time semantic steering need not operate in this manner. Let $F^+$ denote a favored framing and $F^-$ an alternative framing. A steering policy can establish:

\begin{equation}
P_{\theta,\mathcal{I}}(F^+\mid x)
>
P_\theta(F^+\mid x)
\end{equation}

and

\begin{equation}
P_{\theta,\mathcal{I}}(F^-\mid x)
<
P_\theta(F^-\mid x)
\end{equation}

without making $F^-$ impossible.

Consider a public-policy query concerning regulation. Two factually defensible narrative frames may emphasize different dimensions:

\begin{itemize}[leftmargin=*, noitemsep]

    \item \textbf{Safeguard framing:}
    consumer protection, risk mitigation, accountability, and long-term stability;

    \item \textbf{Burden framing:}
    compliance cost, administrative overhead, reduced flexibility, and economic friction.

\end{itemize}

An inference policy does not need to fabricate information to influence the resulting interpretation. It can instead systematically alter which facts, descriptors, examples, and causal relationships become most probable during generation.

\section{Deployment Paradigms and Threat Models}

\subsection{State-Enforced Framing Mandates}

Consider a hypothetical regulatory environment in which authorities require AI intermediaries to promote designated framing guidelines $G$ for selected public-policy topics:

\begin{equation}
z'_t
=
z_t
+
\lambda
s_G(w_t,\mathbf{e}_{\mathrm{state}}).
\end{equation}

The underlying parameters $\theta$ need not change. The deployed system can remain behaviorally ordinary on unrelated prompts while systematically altering interpretive salience on targeted topics.

\begin{figure}[ht]
\centering
\begin{tikzpicture}[
    box/.style={
        rectangle,
        draw=MidnightBlue!80,
        fill=MidnightBlue!5,
        thick,
        minimum height=1.1cm,
        minimum width=2.2cm,
        align=center,
        font=\small
    },
    statebox/.style={
        rectangle,
        draw=BrickRed!80,
        fill=BrickRed!8,
        thick,
        minimum height=1.1cm,
        minimum width=2.4cm,
        align=center,
        font=\small
    },
    arrow/.style={
        -{Stealth[length=2.5mm]},
        thick,
        color=MidnightBlue!80
    }
]
    \node[box] (query) {Citizen Query\\{\footnotesize e.g., Reform}};
    \node[box, right=1.0cm of query] (base) {Base Model\\$P_\theta$};
    \node[statebox, right=1.6cm of base] (mandate)
        {Runtime Policy $\mathcal{I}_G$\\$s_G(w_t,\mathbf{e}_{\mathrm{state}})$};
    \node[box, right=1.4cm of mandate] (output)
        {Framed Output\\$P_{\theta,\mathcal{I}_G}$};

    \draw[arrow] (query) -- (base);
    \draw[arrow] (base) -- (mandate)
        node[midway, above=3pt, font=\scriptsize\bfseries, text=MidnightBlue]
        {Logits};
    \draw[arrow] (mandate) -- (output)
        node[midway, above=3pt, font=\scriptsize\bfseries, text=BrickRed]
        {Served Distribution};
\end{tikzpicture}
\caption{A hypothetical state-enforced inference-steering architecture. Model weights remain unchanged while the deployed system systematically modifies generation at runtime.}
\label{fig:state_threat}
\end{figure}
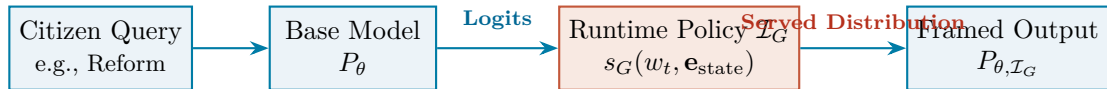

The important governance question is therefore not limited to whether a model was trained on politically biased data. It also includes whether an otherwise unchanged model is embedded within a serving stack containing undeclared behavioral policies.

\subsection{Personalized Persuasion}

Runtime steering can also be conditioned on user or interaction state $\mathbf{u}$.

For illustration, an intervention strength may depend on inferred receptivity:

\begin{equation}
\lambda(\mathbf{u})
=
\begin{cases}
\lambda_{\mathrm{low}},
&
\text{for profiles estimated to resist the target frame},
\\
\lambda_{\mathrm{high}},
&
\text{for profiles estimated to be receptive to the target frame}.
\end{cases}
\end{equation}

This architecture enables individualized persuasive behavior without requiring separate model weights for each target group.

Such a system should not be interpreted as necessarily existing in current production platforms. It is instead a technically feasible deployment pattern implied by the combination of personalization systems, inference-time steering mechanisms, and evidence that conversational framing can influence users \cite{hackenburg2024evaluating,salvi2024conversational}.

\subsection{Commercial Framing: Probability Placement}

We use the term \emph{Probability Placement} to describe a deployment pattern in which commercial influence is embedded into the probability distribution of an otherwise general-purpose conversational assistant.

Suppose a generated sequence is:

\begin{equation}
y=(w_1,\dots,w_T).
\end{equation}

Its probability under a commercially influenced inference policy is:

\begin{equation}
P_{\theta,\mathcal{I}}
(y\mid x)
=
\prod_{t=1}^{T}
P_{\theta,\mathcal{I}}
(
w_t
\mid
x,w_{<t},\mathbf{e}_{\mathrm{commercial}}
).
\end{equation}

The defining characteristic is not merely that an advertiser can influence token probabilities. Token-auction mechanisms already establish that possibility \cite{dutting2024token}.

Instead, Probability Placement concerns the case where a user interacts with what appears to be an organic, general-purpose assistant while an undisclosed commercial policy systematically influences the distribution from which the assistant speaks.

\begin{figure}[ht]
\centering
\begin{tikzpicture}[
    scale=0.85,
    every node/.style={transform shape},
    barbase/.style={fill=MidnightBlue!20, draw=MidnightBlue, thick},
    barsteered/.style={fill=BrickRed!30, draw=BrickRed, thick},
    font=\small
]

    \draw[->, thick] (0,0) -- (6.2,0)
        node[right=2pt, font=\scriptsize] {Candidate};
    \draw[->, thick] (0,0) -- (0,3.6)
        node[above=2pt, font=\scriptsize] {$P(w)$};

    \node[font=\footnotesize\bfseries] at (3.0,3.4)
        {Base Distribution ($P_\theta$)};

    \draw[barbase] (0.3,0) rectangle (1.2,1.70);
    \draw[barbase] (1.7,0) rectangle (2.6,1.55);
    \draw[barbase] (3.1,0) rectangle (4.0,1.35);
    \draw[barbase] (4.5,0) rectangle (5.4,0.40);

    \node[above=2pt,font=\scriptsize] at (0.75,1.70) {0.34};
    \node[above=2pt,font=\scriptsize] at (2.15,1.55) {0.31};
    \node[above=2pt,font=\scriptsize] at (3.55,1.35) {0.27};
    \node[above=2pt,font=\scriptsize] at (4.95,0.40) {0.08};

    \node[below=5pt,font=\scriptsize] at (0.75,0) {Brand A};
    \node[below=5pt,font=\scriptsize] at (2.15,0) {Brand B};
    \node[below=5pt,font=\scriptsize] at (3.55,0) {Brand C};
    \node[below=5pt,font=\scriptsize] at (4.95,0) {Other};

    \begin{scope}[xshift=7.7cm]

        \draw[->, thick] (0,0) -- (6.2,0)
            node[right=2pt, font=\scriptsize] {Candidate};
        \draw[->, thick] (0,0) -- (0,3.6)
            node[above=2pt, font=\scriptsize] {$P_{\theta,\mathcal{I}}(w)$};

        \node[font=\footnotesize\bfseries] at (3.0,3.4)
            {Steered Distribution ($P_{\theta,\mathcal{I}}$)};

        \draw[barsteered] (0.3,0) rectangle (1.2,2.60);
        \draw[barsteered] (1.7,0) rectangle (2.6,1.20);
        \draw[barsteered] (3.1,0) rectangle (4.0,1.00);
        \draw[barsteered] (4.5,0) rectangle (5.4,0.20);

        \node[above=2pt,font=\scriptsize] at (0.75,2.60) {0.52};
        \node[above=2pt,font=\scriptsize] at (2.15,1.20) {0.24};
        \node[above=2pt,font=\scriptsize] at (3.55,1.00) {0.20};
        \node[above=2pt,font=\scriptsize] at (4.95,0.20) {0.04};

        \node[
            below=5pt,
            font=\scriptsize\bfseries,
            text=BrickRed
        ] at (0.75,0) {Brand A*};

        \node[below=5pt,font=\scriptsize] at (2.15,0) {Brand B};
        \node[below=5pt,font=\scriptsize] at (3.55,0) {Brand C};
        \node[below=5pt,font=\scriptsize] at (4.95,0) {Other};

    \end{scope}

\end{tikzpicture}

\caption{Illustrative Probability Placement. The distributions sum to one in both cases. An undisclosed inference policy shifts probability mass toward a commercially preferred entity while leaving competing entities possible. The numerical values are illustrative rather than empirical.}
\label{fig:prob_placement}
\end{figure}
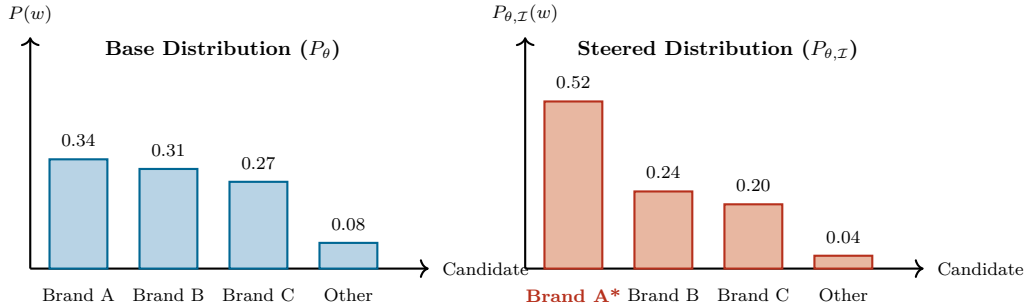

Probability Placement can operate along several dimensions:

\begin{enumerate}[leftmargin=*, noitemsep]

    \item \textbf{Entity selection probability:}
    increasing the probability that a preferred product, company, or service is mentioned or ranked first;

    \item \textbf{Attribute association:}
    increasing the probability that preferred entities are paired with favorable descriptors such as \emph{reliable}, \emph{standard}, or \emph{easy to integrate};

    \item \textbf{Comparative salience:}
    increasing the probability that disadvantages of competing entities are surfaced while equivalent disadvantages of the preferred entity are omitted;

    \item \textbf{Recommendation persistence:}
    repeatedly favoring the same commercial entity across paraphrased or semantically equivalent queries.

\end{enumerate}

Unlike a conventional sponsored result, such an intervention need not produce a visually separable advertising unit. Its commercial effect may instead be embedded in the linguistic judgment presented by the assistant.

\section{The Inference Attribution Problem}

\subsection{Deployed Systems as Composite Functions}

The distribution observed at a production endpoint can be represented abstractly as:

\begin{equation}
P_{\mathrm{deployed}}
=
\mathcal{F}
\left(
\theta_{\mathrm{base}},
\mathcal{D}_{\mathrm{SFT}},
\mathcal{R}_{\mathrm{pref}},
\mathbf{x}_{\mathrm{sys}},
\mathcal{G}_{\mathrm{RAG}},
\mathcal{A}_{\mathrm{activation}},
\mathcal{I}_{\mathrm{logit}},
\mathcal{S}_{\mathrm{sampler}}
\right).
\end{equation}

Here:

\begin{itemize}[leftmargin=*, noitemsep]

    \item $\theta_{\mathrm{base}}$ denotes the underlying model parameters;

    \item $\mathcal{D}_{\mathrm{SFT}}$ denotes supervised fine-tuning effects;

    \item $\mathcal{R}_{\mathrm{pref}}$ denotes preference-optimization mechanisms such as RLHF or DPO;

    \item $\mathbf{x}_{\mathrm{sys}}$ denotes hidden system instructions;

    \item $\mathcal{G}_{\mathrm{RAG}}$ denotes retrieval-augmented context;

    \item $\mathcal{A}_{\mathrm{activation}}$ denotes runtime activation-level interventions;

    \item $\mathcal{I}_{\mathrm{logit}}$ denotes logit-processing policies;

    \item $\mathcal{S}_{\mathrm{sampler}}$ denotes the sampling configuration.

\end{itemize}

A behavioral auditor generally observes only samples from $P_{\mathrm{deployed}}$.

The decomposition responsible for those samples remains latent.

\subsection{Observational Non-Identifiability}

The attribution problem can be expressed more directly.

Let $P_\theta(w\mid x)$ denote a base model distribution for a fixed context $x$, and let $Q(w\mid x)$ denote some target served distribution with support contained in the support of $P_\theta$.

\begin{proposition}[Observational Non-Identifiability of Inference Steering]
\label{prop:nonidentifiability}

For any target distribution $Q(w\mid x)$ satisfying

\begin{equation}
Q(w\mid x)>0
\Rightarrow
P_\theta(w\mid x)>0,
\end{equation}

there exists a logit-level inference policy $\mathcal{I}$ such that

\begin{equation}
P_{\theta,\mathcal{I}}(w\mid x)
=
Q(w\mid x).
\end{equation}

Consequently, the served distribution $Q$ is observationally compatible both with:

\begin{enumerate}[label=(\alph*), noitemsep]

    \item a base model $P_\theta$ combined with a non-trivial inference policy $\mathcal{I}$; and

    \item a different model $P_{\theta'}=Q$ combined with the identity inference policy.

\end{enumerate}

Black-box observations of the served distribution alone therefore cannot uniquely identify whether the observed behavior originates from model parameters or runtime steering.

\end{proposition}

\begin{proof}

Consider the logit transformation:

\begin{equation}
z'(w)
=
z(w)
+
\lambda s(w),
\end{equation}

with

\begin{equation}
s(w)
=
\frac{1}{\lambda}
\log
\frac{
Q(w\mid x)
}{
P_\theta(w\mid x)
}.
\end{equation}

The corresponding served distribution is:

\begin{align}
P_{\theta,\mathcal{I}}(w\mid x)
&\propto
P_\theta(w\mid x)
\exp
\left(
\lambda s(w)
\right)
\\
&=
P_\theta(w\mid x)
\exp
\left(
\log
\frac{
Q(w\mid x)
}{
P_\theta(w\mid x)
}
\right)
\\
&=
Q(w\mid x).
\end{align}

Since $Q$ is already normalized,

\begin{equation}
P_{\theta,\mathcal{I}}(w\mid x)
=
Q(w\mid x).
\end{equation}

Now consider a second system whose base model directly implements

\begin{equation}
P_{\theta'}(w\mid x)=Q(w\mid x)
\end{equation}

and whose runtime inference policy is the identity transformation.

Both systems therefore expose the same observable distribution despite having different internal causal structures.

\end{proof}

The proposition is deliberately simple. Its significance is architectural rather than algorithmic.

If two structurally distinct implementations produce the same observable probability law, no amount of output-only observation can distinguish them without additional assumptions, privileged access, reference execution, instrumentation, or attestation.

This yields the central distinction:

\begin{equation}
\boxed{
\text{Behavioral evidence}
\neq
\text{architectural attribution}
}
\end{equation}

The result extends conceptually beyond logit policies. Hidden prompts, retrieval augmentation, post-training, activation steering, and sampling configuration can similarly create observationally overlapping behavioral distributions.

\subsection{Behavioral Equivalence Classes}

For a deployed distribution $Q$, define an equivalence class of implementations:

\begin{equation}
[\mathcal{M}]_Q
=
\left\{
\mathcal{M}_i :
P_{\mathcal{M}_i}(y\mid x)=Q(y\mid x)
\right\}.
\end{equation}

A black-box auditor observes membership in the behavioral equivalence class but not the specific implementation responsible for the output.

The relevant inference problem is therefore not merely:

\begin{equation}
\text{Does the system exhibit bias?}
\end{equation}

but:

\begin{equation}
\text{Which component of the deployed system causes the observed bias?}
\end{equation}

These are fundamentally different questions.

\subsection{Operational Trade-offs}

Table~\ref{tab:intervention_comparison} summarizes qualitative differences across common steering mechanisms.

\begin{table}[ht]
\centering
\small

\caption{Qualitative comparison of behavioral steering mechanisms across the deployment stack. Exact latency and observability depend on implementation.}

\label{tab:intervention_comparison}

\begin{tabularx}{\textwidth}{@{}
>{\RaggedRight\hsize=1.2\hsize}X
>{\Centering\hsize=0.8\hsize}X
>{\Centering\hsize=1.0\hsize}X
>{\Centering\hsize=1.2\hsize}X
>{\Centering\hsize=1.0\hsize}X@{}}

\toprule

\textbf{Intervention Layer}
&
\textbf{Weight Mutation}
&
\textbf{Context Token Cost}
&
\textbf{Potential Textual Trace}
&
\textbf{Typical Serving Overhead}
\\

\midrule

RLHF / DPO
&
Yes
&
None
&
None
&
None at inference beyond model itself
\\

Hidden System Prompt
&
No
&
Linear in prompt length
&
Possible through extraction or leakage
&
Low
\\

RAG
&
No
&
Linear in retrieved context
&
Possible through retrieved content
&
Low to High
\\

Activation Intervention
&
No
&
None
&
No prompt artifact
&
Low to Moderate
\\

\textbf{Logit Policy $\mathcal{I}$}
&
\textbf{No}
&
\textbf{None}
&
\textbf{No prompt artifact}
&
\textbf{Low to Moderate}
\\

\bottomrule

\end{tabularx}
\end{table}

Logit-level steering has a notable property: it need not leave textual traces inside the model's context window.

This does not imply that it is undetectable. Timing measurements, log-probability access, controlled differential experiments, internal instrumentation, compromised infrastructure, or privileged audit access may expose or constrain the existence of a runtime policy.

The narrower claim is that logit-level policies can alter generation without introducing prompt tokens that an ordinary user can inspect or extract.

\section{Auditing, Detection, and Attribution}

\subsection{Detection Is Not Attribution}

A behavioral audit may establish that:

\begin{equation}
P_{\mathrm{served}}
\neq
P_{\mathrm{reference}}.
\end{equation}

This is evidence of behavioral divergence.

It does not establish:

\begin{equation}
\mathcal{I}_{\mathrm{logit}}
\neq
\operatorname{id}.
\end{equation}

The divergence could instead arise from a different model checkpoint, post-training configuration, system prompt, retrieval policy, activation intervention, or sampling configuration.

Thus:

\begin{equation}
\boxed{
\text{Detection of behavioral shift}
\neq
\text{attribution of intervention locus}
}
\end{equation}

This distinction is particularly important for black-box ideological-bias audits. Such methods can measure systematic behavioral asymmetries \cite{kroger2025dont}, but without privileged architectural information they cannot necessarily determine where within the serving stack those asymmetries originate.

\subsection{Distributional Divergence Metrics}

Suppose a reference execution environment provides $P_{\mathrm{ref}}$ and production exposes $P_{\mathrm{served}}$.

Where token probabilities are accessible, divergence may be quantified using Kullback--Leibler divergence:

\begin{equation}
D_{\mathrm{KL}}
\left(
P_{\mathrm{served}}(\cdot\mid x)
\parallel
P_{\mathrm{ref}}(\cdot\mid x)
\right)
=
\sum_{w\in\mathcal{V}}
P_{\mathrm{served}}(w\mid x)
\log
\frac{
P_{\mathrm{served}}(w\mid x)
}{
P_{\mathrm{ref}}(w\mid x)
}.
\end{equation}

Total Variation distance provides another measure:

\begin{equation}
\delta_{\mathrm{TV}}
\left(
P_{\mathrm{served}},
P_{\mathrm{ref}}
\right)
=
\frac{1}{2}
\sum_{w\in\mathcal{V}}
\left|
P_{\mathrm{served}}(w\mid x)
-
P_{\mathrm{ref}}(w\mid x)
\right|.
\end{equation}

However, many commercial endpoints expose only sampled text.

Under such conditions, auditors may instead estimate semantic distributional shifts across repeated generations. For competing frames $F^+$ and $F^-$, define:

\begin{equation}
\Delta\mathcal{S}(x,F)
=
\mathbb{E}_{y\sim P_{\mathrm{served}}}
\left[
\cos(
\mathcal{E}(y),
\mathcal{E}(F^+)
)
-
\cos(
\mathcal{E}(y),
\mathcal{E}(F^-)
)
\right],
\end{equation}

where $\mathcal{E}(\cdot)$ denotes a validated sentence-level representation.

Repeated measurements across paraphrases, languages, geographic origins, account states, and randomized interaction histories can help identify systematic behavioral asymmetries.

Yet even statistically convincing asymmetry remains evidence about the served system, not necessarily its architectural provenance.

\section{Verifiable Inference and Runtime Transparency}

\subsection{Beyond Behavioral Auditing}

Because black-box observation alone cannot generally resolve implementation-level attribution, stronger governance models may require additional observability.

One approach is to expose cryptographically verifiable information about the inference environment.

A possible \emph{Inference Policy Transparency} framework could combine:

\begin{enumerate}[leftmargin=*, noitemsep]

    \item \textbf{Measured Execution Environments:}
    critical model-serving components execute inside Trusted Execution Environments or other confidential-computing systems capable of remote attestation;

    \item \textbf{Model Identity Attestation:}
    the deployed environment exposes a cryptographic commitment to the model checkpoint or weight set being executed;

    \item \textbf{Inference-Policy Attestation:}
    the system commits to the version or hash of active activation, logit-processing, sampling, and filtering policies;

    \item \textbf{Policy Change Logging:}
    modifications to runtime policies are signed and recorded in a tamper-evident audit log;

    \item \textbf{Independent Policy Review:}
    authorized auditors evaluate whether the attested runtime configuration corresponds to the declared behavioral policy.

\end{enumerate}

Conceptually, an attestation may bind:

\begin{equation}
R
=
\operatorname{Sign}_{K}
\left(
H(\theta),
H(\mathcal{I}),
H(\mathcal{S}),
H(C),
t
\right),
\end{equation}

where $H(\theta)$ denotes the model commitment, $H(\mathcal{I})$ the inference-policy commitment, $H(\mathcal{S})$ the sampler configuration, $H(C)$ the measured serving code, and $t$ a timestamp or execution epoch.

\subsection{The Limits of Attestation}

Attestation does not solve the normative problem by itself.

A trusted execution environment may establish that a specific runtime policy executed. It does not establish that the policy is politically neutral, commercially fair, scientifically justified, or legally permissible.

In other words:

\begin{equation}
\boxed{
\text{Attestation proves execution identity, not policy neutrality.}
}
\end{equation}

Cryptographic verification therefore complements rather than replaces institutional oversight.

The governance value of attestation lies in reducing one dimension of uncertainty: whether the system executed the declared inference stack.

Human, legal, or regulatory evaluation is still required to determine whether the declared policy itself is acceptable.

\section{Regulatory Implications}

\subsection{EU AI Act}

Article~5(1)(a) of the EU AI Act \cite{eu_ai_act_2024} prohibits certain AI practices involving subliminal, purposefully manipulative, or deceptive techniques when the statutory conditions for behavioral distortion and harm are satisfied.

Inference-time framing raises a difficult boundary question.

A subtle probability shift may influence language without producing an individually obvious or immediately measurable injury:

\begin{equation}
\operatorname{Impact}_i
\approx
\epsilon.
\end{equation}

At very large scale, however, repeated effects may aggregate:

\begin{equation}
\sum_{i=1}^{N}
\operatorname{Impact}_i
\gg
\epsilon.
\end{equation}

This does not imply that undisclosed inference steering automatically violates Article~5. Applicability depends on the statutory elements, factual circumstances, purpose of the intervention, affected population, and legally relevant consequences.

The more general governance issue is that probabilistic framing may be difficult to map onto legal frameworks originally designed around more visible forms of manipulation or discrete decision-making.

\subsection{Digital Services Act}

The Digital Services Act \cite{eu_dsa_2022} provides a useful transparency analogy.

Its recommender-system provisions recognize that ranking and information-selection mechanisms can shape what users see even when the underlying content remains available.

Conversational assistants complicate this model because retrieval, ranking, synthesis, framing, and recommendation can be collapsed into a single generated response.

A future transparency regime for conversational systems could therefore require disclosure not only of retrieval or ranking parameters but also of material runtime policies that systematically affect which entities, arguments, or frames are favored during generation.

The claim is not that existing DSA provisions necessarily impose such requirements on every inference-time intervention. Rather, recommender-system transparency provides an institutional model for thinking about generative systems whose outputs are shaped by non-visible selection mechanisms.

\subsection{Commercial Disclosure and Advertising Principles}

Commercial Probability Placement also raises questions familiar from advertising law.

FTC endorsement guidance \cite{ftc_endorsement_2023} emphasizes disclosure where material commercial relationships may affect how consumers interpret endorsements or recommendations.

Traditional advertising generally creates some distinction between editorial content and sponsored content.

Conversational systems can collapse that distinction.

If a general-purpose assistant presents a recommendation in its own unified voice while undisclosed commercial policies alter which products appear, how they are characterized, or how competing products are framed, the relevant governance problem is not merely ad placement but \emph{editorial provenance}.

This creates a potential future disclosure principle:

\begin{quote}
When commercial consideration materially influences a generative system's recommendation distribution, users should be able to distinguish sponsored influence from the system's otherwise organic generation process.
\end{quote}

The precise legal obligations associated with such a principle vary by jurisdiction and deployment context. The broader point is architectural: conventional sponsorship disclosures assume an observable advertising object, whereas Probability Placement may operate within the distribution that constructs the assistant's own narrative.

\section{Discussion}

\subsection{From Model Audits to System Audits}

The Inference Attribution Problem suggests a change in the unit of analysis used by AI auditing.

A model audit asks:

\begin{equation}
\text{What behavior is encoded or elicited by } M_\theta?
\end{equation}

A deployed-system audit asks:

\begin{equation}
\text{What behavior is ultimately produced by the entire serving stack?}
\end{equation}

These questions overlap, but they are not equivalent.

A model checkpoint can behave differently across providers, regions, user cohorts, account states, product tiers, or time periods if surrounding inference policies differ.

Conversely, two distinct model checkpoints may be configured to produce behaviorally similar outputs through runtime interventions.

Therefore:

\begin{equation}
\boxed{
\text{Auditing the model is not auditing the system that speaks.}
}
\end{equation}

\subsection{Temporal Attribution}

Runtime steering also introduces a temporal dimension.

Let the deployed policy be indexed by time:

\begin{equation}
\mathcal{I}_t.
\end{equation}

The same nominal model version can then produce different behavioral distributions at two dates:

\begin{equation}
P_{\theta,\mathcal{I}_{t_1}}
\neq
P_{\theta,\mathcal{I}_{t_2}}.
\end{equation}

Behavioral audits therefore need reproducibility information not only about model identity but also about deployment configuration and time.

A statement such as ``Model X exhibited bias B'' may be underspecified if the relevant behavior was actually produced by a mutable serving environment.

\subsection{Scope and Limitations}

This paper is conceptual and does not establish that major production language-model providers currently deploy undisclosed political or commercial logit-steering mechanisms.

The threat models described here demonstrate architectural feasibility, not evidence of actual misconduct.

Similarly, the non-identifiability result establishes limits on black-box causal attribution in the general case. Specific deployments may expose additional information---such as log probabilities, open weights, reproducible checkpoints, policy documentation, or auditable code---that substantially reduces the attribution problem.

Future empirical work should investigate practical protocols for distinguishing classes of runtime intervention under partial observability.

\section{Future Research}

Several directions follow from this framework.

\subsection{Differential Deployment Auditing}

If auditors can obtain both a reference model and a production endpoint, controlled prompts may reveal systematic divergence:

\begin{equation}
\Delta(x)
=
D(
P_{\mathrm{production}}(\cdot\mid x),
P_{\mathrm{reference}}(\cdot\mid x)
).
\end{equation}

Experiments could test whether divergence concentrates around political topics, commercial entities, demographic attributes, geographic regions, or account-specific features.

\subsection{Counterfactual Brand Audits}

Probability Placement can be evaluated through symmetry tests.

For competing brands $A$ and $B$, an auditor can compare semantically mirrored prompts:

\begin{align}
x_A &= \text{``Compare Brand A with Brand B.''}
\\
x_B &= \text{``Compare Brand B with Brand A.''}
\end{align}

Repeated sampling can estimate:

\begin{equation}
P(A\text{ recommended}\mid x_A,x_B)
\end{equation}

and test whether brand preference persists after controlling for prompt order and factual attributes.

\subsection{Semantic Framing Benchmarks}

Future benchmarks could define paired framing axes such as:

\begin{itemize}[leftmargin=*, noitemsep]

    \item innovation vs.\ risk;

    \item regulation vs.\ burden;

    \item security vs.\ liberty;

    \item labor protection vs.\ labor flexibility;

    \item market leader vs.\ incumbent;

    \item open ecosystem vs.\ fragmented ecosystem.

\end{itemize}

The objective would not be to declare one frame neutral but to measure whether deployment systems systematically and reproducibly privilege one frame over its alternatives.

\subsection{Inference Provenance Standards}

Standardized provenance metadata could eventually complement model cards.

A deployment manifest might identify:

\begin{equation}
\mathcal{M}_{\mathrm{deployment}}
=
\{
H(\theta),
H(\mathbf{x}_{\mathrm{sys}}),
H(\mathcal{G}),
H(\mathcal{A}),
H(\mathcal{I}),
H(\mathcal{S})
\}.
\end{equation}

Such a manifest would not necessarily expose proprietary policy contents publicly. It could instead provide verifiable commitments allowing authorized auditors to establish whether a deployment changed between evaluation and production.

\section{Conclusion}

The decoupling of foundation-model parameters from production serving behavior is a consequential architectural feature of modern language systems.

Controlled generation, activation steering, decoding-time interventions, and statistical watermarking demonstrate that the text observed by users can be systematically modified during inference without requiring changes to the underlying model weights.

This observation leads to the \emph{Inference Attribution Problem}. Behavior observed through a black-box interface does not, in general, uniquely identify the architectural layer responsible for that behavior. Structurally distinct systems can be observationally equivalent at their outputs.

This distinction matters for both empirical auditing and governance.

A behavioral shift can be detected without its causal locus being identified. Commercial influence can potentially be embedded inside an assistant's generated distribution rather than presented as a separable advertising object. Political or institutional framing can theoretically be introduced at deployment time even when the underlying model checkpoint remains unchanged.

We characterize one commercially relevant instance of this phenomenon as \emph{Probability Placement}: undisclosed probability-level influence within an ostensibly organic assistant response. The concept builds on, but is distinct from, prior token-auction mechanisms in which advertisers explicitly participate in generative advertising markets.

These phenomena suggest that governance frameworks should increasingly treat the deployed inference pipeline---not only the foundation model---as the object requiring transparency and auditability.

The central implication is therefore simple:

\begin{equation}
\boxed{
\text{Model behavior is not necessarily deployed-system behavior.}
}
\end{equation}

And consequently:

\begin{equation}
\boxed{
\text{Auditing the model is not auditing the system that speaks.}
}
\end{equation}


\bibliographystyle{unsrt}

\end{document}